\documentclass{l4dc2025}
\usepackage[boxed]{algorithm}
\usepackage[ruled]{algorithm2e}
\usepackage{amsfonts}

\DeclareMathOperator*{\argmax}{arg\,max}
\DeclareMathOperator*{\argmin}{arg\,min}

\newtheorem{problem}{Problem}
\usepackage{amssymb}

\title[Accelerating PPO Learning]{Accelerating Proximal Policy Optimization Learning Using Task Prediction for Solving Environments with Delayed Rewards}
\usepackage{times}

\author{%
 \Name{Ahmad Ahmad} \Email{ahmadgh@bu.edu}\\
 \addr Boston University, Boston, MA, USA
 \AND
 \Name{Mehdi Kermanshah} \Email{mker@bu.edu}\\
 \addr{Boston University, Boston, MA, USA}
 \AND 
  \Name{Kevin Leahy} \Email{kleahy@wpi.edu}\\
 \addr Worcester Polytechnic Institute, Worcester, MA, USA
 \AND
 \Name{Zachary Serlin} \Email{ zachary.serlin@ll.mit.edu}\\
 \addr MIT Lincoln Laboratory, Lexington, MA, USA
  \AND 
  \Name{Ho Chit Siu} \Email{hochit.siu@ll.mit.edu}\\
 \addr MIT Lincoln Laboratory, Lexington, MA
 \AND
 \Name{Makai Mann}
 \Email{mmann@anduril.com}\\
 \addr  Anduril Industries, USA
   \AND 
  \Name{Cristian-Ioan Vasile} \Email{ cvr519@lehigh.edu}\\
 \addr Lehigh University, Bethlehem, PA, USA
 \AND
 \Name{Roberto Tron}
 \Email{tron@bu.edu}\\
 \addr Boston University, Boston, MA, USA.  
  \AND
 \Name{Calin Belta}
 \Email{calin@umd.edu}\\
 \addr  University of Maryland, College Park, MD, USA
 }

\begin{document}

\maketitle

\begin{abstract}%
In this paper, we tackle the challenging problem of delayed rewards in reinforcement learning (RL). While Proximal Policy Optimization (PPO) has emerged as a leading Policy Gradient method, its performance can degrade under delayed rewards. We introduce two key enhancements to PPO: a hybrid policy architecture that combines an offline policy (trained on expert demonstrations) with an online PPO policy, and a reward shaping mechanism using Time Window Temporal Logic (TWTL). The hybrid architecture leverages offline data throughout training while maintaining PPO's theoretical guarantees. Building on the monotonic improvement framework of Trust Region Policy Optimization (TRPO), we prove that our approach ensures improvement over both the offline policy and previous iterations, with a bounded performance gap of $(2\varsigma\gamma\alpha^2)/(1-\gamma)^2$, where $\alpha$ is the mixing parameter, $\gamma$ is the discount factor, and $\varsigma$ bounds the expected advantage. Additionally, we prove that our TWTL-based reward shaping preserves the optimal policy of the original problem. TWTL enables formal translation of temporal objectives into immediate feedback signals that guide learning. We demonstrate the effectiveness of our approach  through extensive experiments on an inverted pendulum and a lunar lander environments, showing improvements in both learning speed and final performance compared to standard PPO and offline-only approaches.
\end{abstract}


\section{Introduction}\label{sec:intro}

Reinforcement learning (RL) in environments with delayed rewards presents a fundamental challenge: actions that lead to successful outcomes may not receive immediate positive feedback, making it difficult for learning algorithms to identify and reinforce beneficial behaviors. This challenge is particularly acute in complex games like soccer, where the value of tactical decisions (e.g., passing plays) may only become apparent after multiple time steps later through their impact on strategic outcomes (e.g., scoring opportunities).
Policy Gradient (PG) methods, particularly Proximal Policy Optimization (PPO) \cite{ppo}, have demonstrated remarkable success in RL tasks. PPO's effectiveness stems from its optimization of a surrogate objective that provides stable policy updates while maximizing expected rewards. However, in settings with delayed rewards, even PPO's carefully constructed optimization landscape becomes difficult to navigate, as the temporal gap between actions and their consequences creates a sparse and uninformative gradient signal.

Our work addresses these challenges through three main contributions. First, we introduce a hybrid policy architecture
that combines an offline policy (trained on expert demonstrations) with an online PPO policy. Unlike previous approaches that use offline data merely for initialization \cite{levine2020awac,ross2012agnostic,schmitt2018kickstarting,chen2021decision,kumar2020conservative}, our method maintains the offline policy as an active guide throughout training. We prove that this architecture guarantees monotonic improvement over both the offline policy and previous iterations. Second, we develop a reward-shaping mechanism using Time Window Temporal Logic (TWTL) \cite{Cristi2017TWTL} that provides immediate, semantically meaningful feedback while preserving the optimal policy of the original problem. This approach bridges the temporal gap in the reward signal by formally encoding desired temporal behaviors. Third, we extend existing convergence proofs for actor-critic methods to handle our hybrid architecture and reward shaping mechanism, providing theoretical foundations for our approach. Our analysis demonstrates that the proposed method maintains PPO's convergence properties, while accelerating learning in delayed-reward settings.

\textbf{Related work.} Our work builds upon three research directions: policy optimization, temporal logics in RL, and offline RL. In policy optimization, Trust Region Policy Optimization (TRPO) \cite{abbele2015TRPO} and PPO \cite{ppo} provides stability guarantees. However, these methods, while effective for standard RL tasks, do not specifically address the challenges of delayed rewards. Our work extends PPO by incorporating temporal logic guidance and offline data while preserving its theoretical guarantees.

Most of the works that have investigated the use of temporal logics (TL) in RL primarily focus on task specification \cite{asarkaya2021twtlRL,Cristi2022overcomingExpl,rewardMachine2022reward,ufuq2020rewardMachine,Ufuq2021advice_guided,ufuq2018safeSheildingRL,jyo2019STL_rewards,Belta2017TLTL,Belta16_QLwithSTL}. TLs formalize high-level tasks into propositional and temporal constraints \cite{baier2008principles}, with some, like Time Window Temporal Logic (TWTL) \cite{Cristi2017TWTL,ahmad2023TWTLrobustness}, handling delayed rewards through temporal constraints \cite{Belta16_QLwithSTL,jyo2019STL_rewards}. TWTL’s clear syntax efficiently expresses sequential tasks. In \cite{asarkaya2021twtlRL}, a Q-learning algorithm learns an optimal policy for TWTL-modeled tasks while maximizing external rewards. Our approach differs by actively using TWTL for reward shaping while maintaining optimality guarantees. Recent works like \cite{Cristi2022overcomingExpl} and \cite{rewardMachine2022reward} have shown promising results in combining TL with deep RL, but primarily focus on task specification rather than reward shaping.

In offline RL, the authors of \cite{Levine23USeOffineDataWithRL, ravari2024implicit} developed algorithms that show the benefits of learning from demonstrations, but typically treats offline data as a fixed dataset rather than an active component of the policy. In \cite{dorsa2023imitation}, the policy chooses between an imitation learning (IL) policy, trained offline, and an online RL policy based on the action with a higher Q-value. Our work differs in that we consider the choosing mechanism as a learnable parameter in the context of deep RL.

\textbf{Paper structure.} The rest of the paper is organized as follows. Sec. \ref{sec:prelims} contains definitions and preliminaries. In Sec. \ref{sec:APPO}, we formally state the problem of learning with delayed rewards using TWTL specifications. Sec. \ref{sec:APPO} introduces our main theoretical contributions: the hybrid policy architecture and TWTL-based reward shaping, along with convergence guarantees. In Sec. \ref{sec:caseStudies}, we present our case study using some benchmark gymnasium environments \cite{gymnasium2023}, namely the inverted pendulum and the lunar lander, including the task predictor architecture and experimental results. The paper concludes with a discussion of limitations and future work in Sec. \ref{sec:conclusion}.

\section{Preliminaries}\label{sec:prelims}



\textbf{Finite-horizon Markov Decision Process (MDP)} 

\begin{definition}\label{def:fintMDP}  A finite horizon MDP is a tuple $(\mathcal{X},\mathcal{U},p(\cdot|\cdot,\cdot),r(\cdot,\cdot),l(.),\mathcal{O})$, where $\mathcal{X}$, $\mathcal{U}$ and 
$\mathcal{O}$ are the state, control, and output spaces, respectively; $p(\cdot|\cdot,\cdot)$ is the state-action pair transition probability; $r:\mathcal{X}\times\mathcal{U}\mapsto [0,1]$ is the reward function; and $l:\mathcal{X}\mapsto \mathcal{O}$ is a labeling function that maps the state to an output observation.
\end{definition}

We denote a state trajectory of the MDP as $\mathbf{x}_{i,i+N}:=x_ix_{i+1}\dots x_{i+N}$, where $x_i\in\mathcal{X}$ and $i\in\mathbb{N}$. $\mathbf{x}_{i,i+N}$ generates a word, $\mathbf{o}_{i,i+N}=o_io_{i+1}\dots o_{i+N}$, where $o_i=l(x_i)$, where $o_i\in\mathcal{O}$. 
Let $\pi:\mathcal{X}\mapsto\mathcal{U}$ be a stochastic policy. In an episodic RL setting \cite{sutton2018reinforcement}, for $K$ learning episodes
the state value function and the state-action value, at iteration $t$ of episode $k$, are defined, respectively, as follows \cite{ppo}:
\begin{equation}\label{eq:v_fn}
    V^{\pi,k}_{t}(x):= \mathop{\mathrm{E}}_{\pi}\left.\left[\sum_{i=t}^{N}r_{i}^k(x_i,u_i)\right|x_t=x\right],
\end{equation}
\begin{equation}\label{eq:q_fn}
   Q^{\pi,k}_{t}(x,u):= \mathop{\mathrm{E}}_{\pi}\left.\left[\sum_{i=t}^Nr_{i}^k(x_i,u_i)\right|x_t=x,\;u_t=u\right].
\end{equation}

where $\mathrm{E}_{\pi}$ is the expectation over the stochastic policy $\pi$. Correspondingly, the advantage function at iteration $t$ of episode $k$:
\begin{equation}\label{eq:adv_fn}
A^{\pi,k}_{t}(x,u):= Q^{\pi,k}_{t}(x,u) - V^{\pi,k}_{t}(x)
\end{equation}
quantifies how much better (or worse) an action $u$ is compared to the average action that would be taken by policy $\pi$ in state $x$. A positive advantage indicates that action $u$ is better than $\pi$'s average action, while a negative advantage suggests it is worse. This function plays a crucial role in policy gradient methods by identifying which actions to encourage or discourage during policy updates.




\textbf{Time Window Temporal Logic} We use TWTL to define tasks. TWTL allows us to specify sequences of tasks, including their order and time constraints. This helps us create reward functions that guide the agent towards specific goals.

An atomic proposition, $\mathrm{AP}\in\Pi$ ($\Pi$ is the set of atomic propositions) is true ($\top$) if $o:=l(x)$ satisfies all components of $h_{\mathrm{AP}}(o)>\sigma_{\mathrm{AP}}$, where $h_{\mathrm{AP}}:\mathcal{O}\mapsto\mathbb{R}^d$ is a predicate function corresponding to the atomic proposition $\mathrm{AP}$ and $\sigma_{\mathrm{AP}}\in \mathbb{R}^d$. Otherwise, it is false ($\bot$). 

The syntax of TWTL is defined, inductively, as follows \cite{Cristi2017TWTL}.
\begin{equation}\label{eq:TWTL_syntax}
 \phi\;::=\;H^{d}s|\phi_1\cdot\phi_2|\phi_1\vee\phi_2 |[\phi]^{[a,b]}
\end{equation} where $s$ is either true, $\top$, or an atomic proposition in $\Pi$ 
;$H^d$ is the \textit{hold} operator where $d\in\mathbb{N}$; $\cdot$ is the \textit{concatenation} operator; $\vee$ is the Boolean conjunction operator; $[]^{[a,b]}$ is the \textit{within} operator, where $d,a,b\in\mathbb{Z}_{\geq0}$ and $a\geq b$; and $\phi_1$ and $\phi_2$ are TWTL formulae.   

The Boolean semantics over an observation word $\mathbf{o}_{t_1,t_2}$, $t_1,t_2\in\mathbb{N}$ and $t_1<t_2$, are defined recursively as follows:
\begin{flalign}
\begin{aligned}\label{eq:twtl_boolean_sat}
    \mathbf{o}_{t_1,t_2}&\models H^{d}s \Leftrightarrow s\in o_t,\;\forall t\in[t_1,t_2]\wedge (t_2-t_1\geq d)  \\
    \mathbf{o}_{t_1,t_2}&\models\phi_1\cdot\phi_2\Leftrightarrow\exists t=\argmin_{t\in[t_1,t_2]}\{\mathbf{o}_{t_1,t}\models\phi_1\}\wedge(\mathbf{o}_{t+1,t_2}\models\phi_2)\\
\mathbf{o}_{t_1,t_2}&\models\phi_1\vee\phi_2\Leftrightarrow(\mathbf{o}_{t_1,t_2}\models\phi_1)\vee(\mathbf{o}_{t_1,t_2}\models\phi_2) \\
\mathbf{o}_{t_1,t_2}&\models[\phi]^{[a,b]}\Leftrightarrow\exists t\geq t_1+a \text{ s.t. } \mathbf{o}_{t,t_1+b}\models\phi\wedge(t_2-t_1\geq b)
\end{aligned}
\end{flalign}

The time horizon of a TWTL formula $\phi$ represents the minimum time required to evaluate whether the formula is satisfied. It is defined as follows. $||\phi|| := \max\{\max(||\phi_1||, ||\phi_2||) \cdot \mathbf{1}_{\phi=\phi_1\vee\phi_2}, (||\phi_1|| + ||\phi_2|| + 1) \cdot \mathbf{1}_{\phi=\phi_1\cdot\phi_2}, d \cdot \mathbf{1}_{\phi=H^ds}, b \cdot \mathbf{1}_{\phi=[\phi_1]^{[a,b]}}\}$, where $\mathbf{1}$ is the indicator function.

\begin{example} [Lunar Lander]\label{ex:lunarLanderTWTLLandingTask}
    The lunar lander environment from Gymnasium (\cite{gymnasium2023}) has state $X = (p_x, p_y, \dot{p}_x, \dot{p}_y, \psi, \dot{\psi}) \in \mathbb{R}^6$, with observation $o=X$, where $(p_x,p_y)$ is the position, $(\dot{p}_x,\dot{p}_y)$ is the velocity, $\psi$ is the angle, and $\dot{\psi}$ is the angular velocity. The control input space $\mathcal{U}$ consists of discrete commands for the main engine and side thrusters. Based on the environment's success criteria, we define the landing task using TWTL:  $\phi_{\mathrm{landing}}:= [H^{100} \mathrm{AP}_{\mathrm{hover}}]^{[0,\;150]}\;\cdot\; [H^{150} \mathrm{AP}_{\mathrm{align}}]^{[100,\;300]}\;\cdot\; [H^{150} \mathrm{AP}_{\mathrm{descend}}]^{[250,\;450]}\;\cdot\; [H^{50} \mathrm{AP}_{\mathrm{land}}]^{[400,\;500]}$; with predicate functions with fixed parameters motivated by the environment's success criteria: $h_{\mathrm{AP}_{\mathrm{hover}}}(o) = \min\{ h_0 - 0.8h_0 - |p_y - 0.8h_0|, 0.1 - |\dot{p}_y| \} \geq 0$; $
h_{\mathrm{AP}_{\mathrm{align}}}(o) = \min\{ 0.2 - |p_x|, 0.1 - |\psi| \} \geq 0$; $h_{\mathrm{AP}_{\mathrm{descend}}}(o) = \min\{ -0.2 - \dot{p}_y, \dot{p}_y + 0.5, 0.15 - |\psi| \} \geq 0$; and
$h_{\mathrm{AP}_{\mathrm{land}}}(o) = \min\{ 0.1 - \sqrt{\dot{p}^2_y+\dot{p}^2_x}, 0.1 - |\psi|, \mathbf{1}_{p_y \leq 0} \} \geq 0$. The time horizon $||\phi_{\mathrm{landing}}||=1403$. The formula $\phi_{\mathrm{landing}}$ reads:
 ``\textit{Within} time $0$ and time $150$, the lander must be hovering for a $100$ time steps (subformula $H^{100}\;\mathrm{AP}_{\mathrm{hover}}$). Then, \textit{within} time $100$ and time $300$, the lander must be aligned with the landing position for a $150$ time steps (subformula $H^{150}\;\mathrm{AP}_{\mathrm{align}}$). After ensuring that the lander is aligned, descend and then dwelling in the landing position." We use the concatenation operator to ensure the correct landing sequence, which is, in high level, \textit{hovering}, \textit{aligning}, \textit{descending}, then \textit{landing}.   
\end{example}

 A formula is feasible if the time window of each \textit{within} operator in the formula is longer than the enclosed task (expressed via the \textit{Hold} operators) (Definition IV.1 in \cite{Cristi2017TWTL}). Let $\Phi$ be the set of feasible TWTL formulas.

 In the following, we define a TWTL robustness which measures how close or far a sequence of observations, $\mathbf{o}_{i_1,i_2}$, is from satisfying the TWTL task. A positive value indicates task satisfaction, with higher values signifying greater robustness. Conversely, a negative value indicates task violation \cite{ahmad2023TWTLrobustness}. 

\begin{definition}(TWTL Robustness)\label{def:traditional_robustness}
Given a TWTL formula $\phi$ and an output word $\mathbf{o}_{i_1,i_2}$ of MDP (Def. \ref{def:fintMDP}), we define the robustness degree $\rho(\mathbf{o}_{i_1,i_2},\phi)$ at time $0$, recursively, as follows:
\begin{flalign}
\begin{aligned}\label{eq:TWTL_robustness}
    & \varrho(\mathbf{o}_{i_1,i_2},H^{d}\mathrm{AP}_{\mathrm{A}}) :=\begin{cases}
    \min\limits_{t\in[i_1,d+i_1]}h(o_t) & ; (i_2-i_1\geq d)         \\
     -\infty & ;\text{otherwise}
    \end{cases}\\
     &\varrho(\mathbf{o}_{i_1,i_2},\phi_1\vee\phi_2):=\max\{ \varrho(\mathbf{o}_{i_1,i_2},\phi_1),\varrho(\mathbf{o}_{i_1,i_2},\phi_2)\}  \\
     & \varrho(\mathbf{o}_{i_1,i_2},\phi_1\cdot\phi_2) :=\max_{i\in [i_1, i_2)} \left\{\min\{\varrho(\mathbf{o}_{i_1,i},\phi_1),\varrho(\mathbf{o}_{i+1,i_2},\phi_2) \} \right\}\\
    &\varrho(\mathbf{o}_{i_1,i_2},[\phi]^{[a,b]}) :=\begin{cases}
    \max\limits_{i\geq i_1+a}\{\varrho(\mathbf{o}_{i,i_1+b},\phi)\};(i_2-i_1\geq b)
    \\
    -\infty;\; \text{otherwise}
    \end{cases}
\end{aligned}
\end{flalign}
\end{definition}

\section{Problem Formulation}\label{sec:probForm}

We consider a model-free RL setting where an agent interacts with a complex environment with delayed rewards. To address the temporal credit assignment challenge, we introduce a reward function based on Time Window Temporal Logic (TWTL) specifications that provides more immediate feedback about task progress.

\begin{definition}[Concrete Time Reward]\label{def:concrete-time-reward}
Let $\phi\in\Phi$ be a TWTL formula. For a finite-horizon MDP with a trajectory $\mathbf{x}_{i,i+||\phi||}$ produced by applying $\mathbf{u}_{i-1,i+||\phi||-1}:= u_{i-1}u_{t}\dots u_{N-i+||\phi||-1}$, we define an episodic concrete-time reward, over the generated observation word $\mathbf{o}_{i,i+||\phi||}$, at time $i$, as follows:
    \begin{equation}\label{eq:concrete_time_reward}
        r_{\phi,t}(\mathbf{o}_{t,t+||\phi||}) := \begin{cases}
        1\qquad \textbf{if}\; \mathbf{o}_{t,t+||\phi||}\models\phi\\
        0\qquad \textbf{if}\; \mathbf{o}_{t,t+||\phi||}\not\models\phi
    \end{cases}
    \end{equation}
\end{definition}
where $k$ denotes the training episode, and $t$ denotes the time instance at which we observe the reward.



\begin{problem}\label{pr:opt_mp_problem}
Given a finite horizon MDP with an episodic concrete-time reward $r_{\phi,t}^k$, compute the optimal parameter $\theta^\ast$ of the stochastic policy, $\pi_{\theta}$, that maximizes the total expected episodic reward. I.e.,
    \begin{equation}\label{eq:Opt_problem1}
        \pi_{\theta^\ast} = \argmax_{\theta\in\Theta} \;\mathop{\mathrm{E}}_{\pi_\theta}\left[\sum_{t=1}^Nr^k_{\phi,t}\right]
    \end{equation}
\end{problem}


\section{Accelerated Proximal Policy Optimization}\label{sec:APPO}
We enhance PPO with an offline policy and a reward-shaping function based on TWTL. The offline policy improves performance, while the reward shaping guides learning towards desired temporal goals.



\subsection{Proximal Policy Optimization}\label{sbusec:PPO}
PPO is a policy optimization algorithm that uses the policy gradient to optimize a parameterized policy. PPO provides stability to the learning process. At each iteration, the algorithm aims to find a better policy that is close to the previous iteration \cite{ppo}. This, in turn, helps keep the learning process away from degenerate policies.

Consider the state and state-action value functions, (\ref{eq:v_fn}) ) --and subsequently the advantage function (\ref{eq:adv_fn})-- to be defined over the \textit{concrete time reward} (\ref{eq:concrete_time_reward}). For a parameterized policy $\pi_\theta$, where $\theta\in\Theta$ is the parameter and $\Theta$ is the parameter space, PPO optimizes the policy at each iteration according to  $\theta_{i+1}\gets \argmax\limits_{\theta\in\Theta}\; \mathop{\mathrm{E}}[J^k(x,u,\theta,\theta_i)]$,
where $J^k(x,u,\theta,\theta_i)$ is a clipped objective and defined as follows \cite{ppo}:
\begin{equation}\label{eq:L_k-Clipped_Objective}
    \begin{aligned}
    J^k(x,u,\theta,\theta_i) := \min \left[ \frac{\pi_\theta(u|x)}{\pi_{\theta_i}(u|x)} \hat{A}^{\pi_{\theta_i},k}_t(u,x),\;g(\epsilon,\hat{A}^{\pi_{\theta_i},k}_i(u,x))\right],
    \end{aligned}
\end{equation}
where $g(\epsilon,\hat{A}):=\begin{cases}(1+\epsilon)\hat{A};\; \hat{A}\geq 0 \\ (1-\epsilon)\hat{A}; \; \hat{A} < 0 \end{cases}$, and $\epsilon\in(0,1)$ is a tunable hyperparameter.

The value $\hat{A}^{\pi_{\theta_i},k}_t$ estimates the advantage $ A^{\pi_{\theta_i},k}_t$ at time $t$ and can be computed as follows \cite{abbeel16Adv_estimate}:
\begin{equation}\label{eq:A_hat}
\begin{aligned}
    \hat{A}^{\pi_{\theta_i},k}_t =\delta_{t}^{k}+(\gamma\lambda) \delta_{t+1}^{k}+\dots+(\gamma\lambda)^{N-t+1} \delta_{N-1}^{k}
\end{aligned},
\end{equation}
where $\delta_{t}^{k}$ is the temporal difference (TD) residual of $V^{\pi_{\theta_i},k}$ discounted by $\gamma$, and is defined as $\delta_t^k :=r_t^k+\gamma V_{t}^{\pi_{\theta_i},k} - V_{t+1}^{\pi_{\theta_i},k}$. In our setting the value function (\ref{eq:v_fn}) is parameterized by a neural network.

The choice of maximizing the objective (\ref{eq:L_k-Clipped_Objective}) implies maximizing the advantage function \cite{2022mirrorLearningPPO}. This choice yields low variance in the gradient; where the advantage function measures how much better or worse the policy is from the default control \cite{abbeel16Adv_estimate}.

\subsection{Temporal Logic Reward-Shapping}
Consider an episodic RL framework where tasks are formulated using TWTL formulae. To simplify the analysis, we consider a single task $\phi$, with time Horizon $||\phi||< N$, where $N$ is the number of steps in the training episode. 

Recall the concrete time reward (\ref{eq:concrete_time_reward}), computing this reward requires complete MDP trajectories that span the entire task horizon, $||\phi||$.  However, obtaining these full trajectories during real-world implementation might be impractical.  To overcome this challenge, we need a method to complete or extend partially observed trajectories. Existing approaches for trajectory completion include trajectory prediction \cite{pavone2020trajectron++}, or runtime monitoring techniques \cite{ahmad2023TWTLrobustness}.

We propose using an observation predictor. This predictor takes an MDP state and generates a sequence of predicted observations spanning the entire task horizon ($||\phi||$). We present Long Short-Term Memory (LSTM)-based task predictor in Example \ref{ex:LSTM_for_LunarLander}. For formulation simplicity, assume we have a pre-defined predictor function, \(\texttt{Pred}: \mathcal{X} \mapsto \mathcal{O}\), 
which maps a state $x_t$ to a sequence of observations 
$(\hat{o}_t, \hat{o}_{t+1}, \dots, \hat{o}_{t+||\phi||}) $.

Consider the TWTL robustness function, $\varrho:\mathcal{O}\times\Phi\mapsto\mathbb{R}$, as defined in (\ref{eq:TWTL_robustness}). For a TWTL task $\phi$, we define a reward shaping function $F:\mathcal{X}\times\mathcal{U}\times\mathcal{X}\times\Phi\mapsto\mathbb{R}$, with $0<\kappa<1$, as the following: 
\begin{equation}\label{eq:F_potential_basedShaping}
    F(x_{t},u_{t},x_{t+1},\phi):=\kappa\cdot \varrho(\texttt{Pred}(x_{t}),\phi) - \varrho(\texttt{Pred}(x_{t+1}),\phi)   
\end{equation}

\begin{lemma}\label{lemma:rewardShaping}
    The optimal policy of the original MDP is the same as the optimal policy of the MDP with the shaped reward function ( $ r_{\phi,t}^{\prime k} := r_{\phi,t}^{k} + F$).
\end{lemma}

\begin{proof}
By Theorem 1 in \cite{ng1999rewardShaping}, $F$ being a potential function, as defined in (\ref{eq:F_potential_basedShaping}), guarantees optimal policy consistency.
\end{proof}  

\begin{example}[Continue]\label{ex:LSTM_for_LunarLander}
For the Lunar Lander environment, we implement a sequence prediction network to forecast future state trajectories, which is similar to \textit{Trajectron ++} \cite{pavone2020trajectron++} but for a single agent case, which are then used to evaluate TWTL robustness. The network architecture consists of: $(i)$ an input embedding layer, $\texttt{embed}$, that processes a window, $w\in\mathbb{N}_{>0}$, of state vectors $\mathbf{x}_{t-w:t} \in \mathbb{R}^{8\times w}$; $(ii)$ an LSTM encoder $\texttt{LSTM}_{\mathrm{e}}$ and decoder $\texttt{LSTM}_{\mathrm{d}}$ for temporal sequence modeling; and $(iii)$ prediction heads for future state estimation. Given a window of states $\mathbf{x}_{t-w:t}$, the network predicts future state trajectories: $\hat{\mathbf{x}}_{t:t+||\phi||} = \texttt{LSTM}_{\mathrm{d}}(\texttt{LSTM}_{\mathrm{e}}(\texttt{embed}(\mathbf{x}_{t-w:t})))$, where $\hat{\mathbf{x}}_{t:t+||\phi||}$ represents the predicted state sequence over the TWTL formula horizon.

The predictor is trained on expert demonstrations to minimize the mean squared error of state predictions, $  \mathcal{L} = \sum_{i=t}^{t+||\phi||} \|\hat{\mathbf{x}}_i - \mathbf{x}_i\|^2$
Using these predicted state trajectories, we can compute the TWTL robustness degree $\varrho$ for the landing task formula $\phi_{\text{landing}}$ by evaluating the atomic predicates (hover, align, descend, land) on the predicted states. This robustness value is then used in the reward shaping function (\ref{eq:F_potential_basedShaping}) to provide immediate feedback about the predicted satisfaction or violation of the landing requirements.
\end{example}

\subsection{Online-Offline Policy Architecture}\label{subsec:Online-offlinePolicyArch}

In this section, we present two key theoretical results for policy improvement in episodic learning. First, we introduce a novel policy architecture that combines an offline policy ($\pi_\rho$) with an online policy ($\pi_\beta$) and prove its convergence properties. Second, we establish guarantees for iterative improvement of this mixed policy over time. Both results build on the same theoretical foundation while addressing different aspects of the learning process.

The core idea of our architecture is to combine an offline policy, pre-trained on expert demonstrations, with an online policy continuously optimized by PPO. Initially, the offline policy guides the online policy's learning, while a mixing parameter gradually reduces the offline policy's influence, allowing the online policy to take over action selection.

Let $\mathrm{D}_{\mathcal{U}}$ be the set of distributions over $\mathcal{U}$. We introduce a deep policy architecture, $\pi_\theta\in\mathrm{D}_{\mathcal{U}}$, that we constitute using two parallel deep policies $\pi_\rho,\pi_\beta\in\mathrm{D}_{\mathcal{U}}$ and combine them using a fully connect layer (FCL), see Figure \ref{fig:WholeSchemeArch}. We compute $\pi_\theta$ as follows.  
\begin{equation}\label{eq:pi_theta_mixing}
    \pi_\theta(u|x) := (1-\alpha) \cdot \pi_\rho(u|x) + \alpha \cdot \pi_\beta(u|x)
\end{equation}

where $\alpha$ is mixing parameters which are the weights of the FCL of the architecture. 

In the following, we demonstrate that policy $\pi_\theta$ is guaranteed to improve upon the offline policy $\pi_\rho$ at every PPO iteration. After that, and using similar arguments, we prove that policy $\pi_\theta$ improves at every iteration. 

We interpret $\pi_\rho$ as a fixed prior of policy $\pi_\theta$, where we aim to learn the optimal $\pi_\beta$ and the optimal mixing parameter $\alpha$ through optimizing the weights of the FCL.

 We base our analysis on the works of \cite{kakade2002approximately} and \cite{abbele2015TRPO}. 
 The total return of the policy $\pi$, at episode $k$, is given by: $\eta^k(\pi) := \mathop{\mathrm{E}}_{u\sim\pi}\left[\sum_{i=1}^{N}r_{i}^{k}(x_i,u_i)\right]$, we drop the superscript $k$ from $\eta$ for notation simplicity.

\begin{lemma}\label{lemma:eta'_eta+E}
In our framework, where we mix an offline policy $\pi_\rho$ with an online policy $\pi_\beta$, we can express the return of the mixed policy $\pi_\theta$ in terms of the offline policy return. This relationship follows a similar structure to policy improvement bounds established in \cite{kakade2002approximately,abbele2015TRPO}. Specifically, as follows.  
\begin{equation}\label{eq:eta_difference}
    \eta(\pi_{\theta}) = \eta(\pi_{\rho}) +\mathop{\mathrm{E}}_{u\sim\pi_{\theta}}\left[\sum_{i=1}^{N}A_{i}^{\pi_\rho}\right]
\end{equation}
\end{lemma}
\begin{proof}
Proof follows directly from Lemma 1 in \cite{abbele2015TRPO} by applying their policy improvement bound to our mixed policy formulation. Specifically, we replace their policy $\pi^\prime$ with our mixed policy $\pi_\theta$, and their baseline policy $\pi$ with our offline policy $\pi_\rho$ in their original proof. The lemma holds since the fundamental relationship between policies' advantage functions and expected returns remains unchanged under our policy mixing architecture.
\end{proof}

 Lemma \ref{lemma:eta'_eta+E} provides the basic formula to compute an improvement bound on $\pi_\theta$ upon $\pi_\rho$, however, it the expectation depends on $\pi_\theta$ which makes it hard to derive an update rule during optimization. Hence, we use a local approximation of the total return, $L:\mathrm{D}_{\mathcal{U}}\mapsto\mathbb{R}$, introduced by Schulman et al. \cite{abbele2015TRPO}, which we require for introduce a bound guaranteeing that $\pi_\theta$ improves upon $\pi_\rho$. Let $u\sim\pi_\rho$ and the visitation frequency $P_{\pi_\rho}(x)=p(x_0 = x|x,u)+p(x_1 = x|x,u)+\dots+p(x_{N} = x|x,u)$, the return estimation of the offline policy, $\pi_\rho$, as a function of a general policy $\pi$ is given as follows.
\begin{equation}\label{eq:L_pirho_pi}
    L_{\pi_{\rho}}(\pi)= \eta(\pi_\rho)+\sum_{x\in\mathcal{X}}P_{\pi_\rho}(x)
\sum_{u\in\mathcal{U}}\pi(u|x)A_{0}^{\pi_\rho}(x,u)
\end{equation} 
Choose $\pi_\beta^\ast = \argmax\limits_{\pi_\beta}L_{\pi_\rho}(\pi_\beta)$. In the following theorem, we represent a fundamental bound to demonstrate the effectiveness of policy architecture (\ref{eq:pi_theta_mixing}). 

Let the expected advantage of the $\pi_\theta$ over the $\pi_\rho$ at state $x$ as follows.  
\begin{equation}\label{eq:Abar}
    \Bar{A}_{i}(x) := \mathop{\mathrm{E}}_{u_i\sim\pi_\theta}\left[\sum_{i=1}^{N}A_{i}^{\pi_\rho}\right]
\end{equation}
\begin{lemma}[\cite{abbele2015TRPO}]\label{lemma:pi_theta_better_pi_rho} For $\alpha$-coupled policies, $\pi_\theta$ (as defined in (\ref{eq:pi_theta_mixing})) and $\pi_\rho$, the following inequity holds: 
\begin{equation}\label{eq:Abar_ineq}
    |\Bar{A}_i(x)|\leq 2\alpha \max\limits_{x\in\mathcal{X},u\in\mathcal{U}}|A_0^{\pi_\beta}(x,u)|
\end{equation}
\end{lemma}

Proposition \ref{prop:BoundOnthetaVsrho} implies that the mixing policy, $\pi_\theta$ improves upon $\pi_\rho$ at every iteration.
\begin{proposition}\label{prop:BoundOnthetaVsrho}
Consider the total return of $\pi_\theta$, $\eta(\pi_\theta)$, and the estimated performance of $\pi_\rho$, $L_{\pi_\rho}$, the following bound holds.
\begin{equation}\label{eq:finalBoundOnPerformance}
  \begin{aligned}
      \eta(\pi_{\theta})\geq L_{\pi_\rho}(\pi_\theta) - \frac{2\varsigma\gamma\alpha^2}{(1-\gamma)^2}
  \end{aligned}
\end{equation}
    where $\varsigma = \max_{x\in\mathcal{X}}|\mathop{\mathrm{E}}_{u\sim\pi_\beta^\ast}[A_{0}^{\pi_\rho}(x,u)]|$, and recall that $\gamma$ is a discount factor the GAE (\ref{eq:A_hat}). 
\end{proposition}

\begin{proof}
Building on Theorems 1 in \cite{abbele2015TRPO} and 4.1 in \cite{kakade2002approximately}, we establish Proposition~\ref{prop:BoundOnthetaVsrho}. Similar to \cite{abbele2015TRPO}. 

Equations (\ref{eq:eta_difference}) can be rewritten as:
\begin{equation}\label{eq:eta_rho_barA}
        \eta(\pi_{\theta}) = \eta(\pi_{\rho}) +\mathop{\mathrm{E}}_{u\sim\pi_{\theta}}\left[\sum_{i=1}^{N}\Bar{A}_{i}(x)\right]
\end{equation}

Correspondingly $L_{\pi_\rho}$ becomes the following:
\begin{equation}\label{eq:L_pi_beta}
    L_{\pi_{\rho}}(\pi_{\theta}):= \eta(\pi_\theta)+\mathop{\mathrm{E}}_{u\sim\pi_{\theta}}\left[\sum_{i=1}^{N}\Bar{A}_{i}(x)\right]
\end{equation}

With a slight abuse of notation, we assign $u\sim\pi_{\theta}$ and $u\sim\pi_{\rho}$ as $u_{\theta}$, $u_{\rho}$, respectively. The probability of event $\{u_\theta=u_\rho\}$ is as follows: $p_c(x) := \sum_{u} \pi_\theta(u|x) \cdot \pi_\rho(u|x)$. According to (\ref{eq:pi_theta_mixing}) and the coupling probability $p_c(x)$, $\pi_\theta$ and $\pi_\rho$ are $\alpha$-coupled, that is:  

\begin{equation}\label{eq:Coupleing_prob_bound}
    P(u_\theta\neq u_\rho|x)\leq \alpha,\;\forall x\in\mathcal{X}
\end{equation}

Given Lemma \ref{lemma:pi_theta_better_pi_rho} and based on Lemma 3 in \cite{abbele2015TRPO}, we have the following:
\begin{equation}\label{eq:BoundOnEbarA}
\begin{aligned}
\mathop{\mathrm{E}}_{u_i\sim\pi_{\theta}}\left[\Bar{A}_i(x_i)\right]-\mathop{\mathrm{E}}_{u_i\sim\pi_{\rho}}\left[\Bar{A}_i(x_i)\right]\leq2\alpha\Bar{A}_i(x_i)
\leq 4\alpha(1-(1-\alpha)^i)\cdot\max\limits_{x\in\mathcal{X}}A_{i}^{\pi_\beta,k}(x_i,u_i)
\end{aligned}
\end{equation}

Subtracting (\ref{eq:eta_rho_barA}) and (\ref{eq:L_pi_beta}) yields the following:
\begin{equation}\label{}
    \begin{aligned}
         \eta(\pi_{\theta})-L_{\pi_\rho}(\pi_\theta)= \mathop{\mathrm{E}}_{u\sim\pi_{\theta}}\left[\sum_{i=1}^{N}\Bar{A}_{i}^{k}(x)\right] - \mathop{\mathrm{E}}_{u\sim\pi_{\rho}}\left[\sum_{i=1}^{N}\Bar{A}_{i}^{k}(x)\right]
       =\sum_{i=1}^{N}\mathop{\mathrm{E}}_{u\sim\pi_{\theta}}\left[\Bar{A}_{i}^{k}(x)\right] -\mathop{\mathrm{E}}_{u\sim\pi_{\rho}}\left[\Bar{A}_{i}^{k}(x)\right] 
    \end{aligned}
\end{equation}
then, based on bound (\ref{eq:BoundOnEbarA}), we get the following bound on the performance between using $\pi_\rho$ and $\pi_\theta$
\begin{equation}\label{eq:prefinalBoundOnPerformance}
  \begin{aligned}
      |\eta(\pi_{\theta}) - L_{\pi_\rho}(\pi_\theta)|\leq \sum_{i=1}^{N}4\alpha(1-(1-\alpha)^i)\cdot\varsigma
  \end{aligned}
\end{equation} 

where $\varsigma = \max\limits_{x\in\mathcal{X}}|\mathop{\mathrm{E}}_{u\sim\pi_\beta}[A_{0}^{\pi_\rho}(x,u)]|$. In order to simplify our analysis, we consider unlimited computational resources and we let the number of iterations $N$ to be $\infty$ and we change the r.h.s. of inequality (\ref{eq:prefinalBoundOnPerformance}) to be $\sum_{i=1}^{\infty}\gamma^i4\alpha(1-(1-\alpha)^i)\cdot\varsigma$, where $\gamma\in (0,1)$ is a discount factor. Inequality (\ref{eq:prefinalBoundOnPerformance}) becomes: 

\begin{equation}\label{eq:finalBoundOnPerformance_absluteValue}
  \begin{aligned}
     | \eta(\pi_{\theta}) - L_{\pi_\rho}(\pi_\theta)|\leq \sum_{i=1}^{\infty}\gamma^i4\alpha(1-(1-\alpha)^i)\cdot\varsigma=\frac{4\varepsilon \alpha}{(1-\gamma)(1-\gamma(1-\alpha))}\leq \frac{4\varepsilon \alpha}{(1-\gamma)^2}
  \end{aligned}
\end{equation} 
since $\frac{4\varepsilon \alpha}{(1-\gamma)^2}\geq 0$. Given the definition of the absolute value function, for the left-hand side, we consider the case that $\eta(\pi_{\theta}) \leq L_{\pi_\rho}(\pi_\theta)$, which represents the worst case on the total return of policy $\pi_\theta$. Then, $L_{\pi_\rho}(\pi_\theta) - \eta(\pi_{\theta})\leq\frac{4\varepsilon \alpha}{(1-\gamma)^2}$, which is a rearrangement of inequality (\ref{eq:finalBoundOnPerformance}).
\end{proof}

In the following, we derive the main result of the paper and the update rule of policy $\pi_\theta$ (defined in Equation \ref{eq:pi_theta_mixing}). First, we consider the policy at iteration $i+1$ to be given as the following: $\pi_{\theta_{i+1}}(u|x) := (1-\mathrm{TV}_{\theta_i}^{\theta_{i+1}} ) \cdot \pi_{\theta_{i}}(u|x) +  \mathrm{TV}_{\theta_i}^{\theta_{i+1}} \cdot \tilde{\pi}^\prime_{\theta_{i+1}}(u|x)$, where $\mathrm{TV}_{\theta_i}^{\theta_{i+1}} = \max\limits_{x\in\mathcal{X}}\mathrm{D}_{\mathrm{TV}}(\theta_{i}||\theta_{i+1})$;  $\mathrm{D}_{\mathrm{TV}}(\theta_{i}||\theta_{i+1}):=\frac{1}{2}\sum\limits_{j}|(u_{j}\sim\pi_{\theta_{i}}(x))-(u^\prime_{j}\sim\pi_{\theta_{i+1}}(x))|$ is the total variation between $\pi_{\theta_{i}}$ and $\pi_{\theta_{i+1}}$, and $\tilde{\pi}^\prime_{\theta_{i+1}}(u|x):=\argmax\limits_{\pi^\prime_{\theta_{i+1}}}L_{\pi_{\theta_{i}}}(\pi^\prime_{\theta_{i+1}})$.



In the next Theorem, we show that the policy $\pi_\theta$ (defined in Equation \ref{eq:pi_theta_mixing}) consistently improves with each iteration. This conclusion is a direct consequence of Proposition \ref{prop:BoundOnthetaVsrho}.

\begin{theorem}
    \label{colry:pi_theta_p1Improvment}
Consider $\pi_{\theta_i}$, as in (\ref{eq:pi_theta_mixing}), with $\theta_i=(\beta_i,\alpha_i)$, where the $i$ stands for the optimization iteration. Then for policy $\pi_{\theta_{i+1}}$, $\pi_{\theta_{i}}$ and $\varsigma = \max\limits_{x\in\mathcal{X}}|\mathop{\mathrm{E}}_{u\sim\pi_\beta^\ast}[A_{0}^{\pi_{\theta_i}}(x,u)]|$, the following bound holds:
\begin{equation}\label{eq:BoundOn_pii1_pii}
  \begin{aligned}
      \eta(\pi_{\theta_{i+1}})\geq L_{\pi_{\theta_{i}}}(\eta(\pi_{\theta_{i+1}})) - \frac{2\varsigma\gamma\cdot(\mathrm{TV}_{\theta_i}^{\theta_{i+1}})^2}{(1-\gamma)^2}
  \end{aligned}
\end{equation}
\end{theorem}
\begin{proof}
The proof follows the same logical structure as Theorem 6, with the key difference being that we apply the analysis to successive policy iterations $\pi_{\theta_i}$ and $\pi_{\theta_{i+1}}$ rather than comparing offline and online policies. By applying the $\alpha$-coupling argument to $\mathrm{TV}_{\theta_i}^{\theta_{i+1}}$ and following the same steps used in Proposition \ref{prop:BoundOnthetaVsrho}'s proof, we establish that the performance difference between successive policies is bounded by the given term involving the total variation distance. This shows that the iterative policy updates maintain improvement guarantees similar to those demonstrated for the offline-online mixing case.
\end{proof}

We use the sampling-based estimation, similar to the original PPO \cite{ppo} and TRPO \cite{abbele2015TRPO} algorithms, which yield the surrogate clipped objective (\ref{eq:L_k-Clipped_Objective}) and is defined based on policy architecture (\ref{eq:pi_theta_mixing}).  

Considering the general class of majorization maximization (MM) algorithms \cite{hunter2004tutorialMM}, maximizing the r.h.s. of (\ref{eq:BoundOn_pii1_pii}) implies maximizing the total return of $\pi_{\theta_{i+1}}$ \cite{abbele2015TRPO}. Similar to (\ref{eq:L_pirho_pi}), with a slight abuse of notation we define the return $L_{\pi_{\theta_{i+1}}}$ as a function of parameter $\theta$: 
\begin{equation}\label{eq:L_pithetai1_pi_samplingFreq}
    L_{\theta_{i}}(\theta_{i+1})= \eta(\theta_{i})+\sum_{x\in\mathcal{X}}P_{\theta_{i}}(x)
\sum_{u\in\mathcal{U}}\pi_{\theta_{i+1}}(u|x)A_{0}^{\theta_{i}}(x,u)
\end{equation}

\subsection{Algorithm Details}\label{subsec:TheAlgorithm}
We need to find $\theta_{i+1}$ that maximizes $L_{\theta_i}(\theta_{i+1})$. We use importance sampling (IS) to estimate $\theta_{i+1}$ that maximizes the term $\sum_{x\in\mathcal{X}}P_{\theta_{i}}(x)
\sum_{u\in\mathcal{U}}\pi_{\theta_{i+1}}(u|x)A_{0}^{\theta_{i}}(x,u)$ in (\ref{eq:L_pithetai1_pi_samplingFreq}) which implies maximizing $L_{\theta_i}(\theta_{i+1})$. By using IS and including penalized $\mathrm{TV}_{\theta_i}^{\theta_{i+1}}$ in the objective that we're maximizing, we conclude to use a clipped objective function that has the same structure as the objective of PPO \cite{ppo}, $J^k(x,u,\theta,\theta_t)$ (see Eq (\ref{eq:L_k-Clipped_Objective})), but is defined with respect to (\ref{eq:pi_theta_mixing}).

Following the formulation of \cite{ICML20effExpl_inPO_OPPO} we introduce APPO in Algorithm \ref{alg:APPO}. Given the offline policy, $\pi_{\rho}$, we initialize for episode $1$ a series of the policy, $\pi_{\theta}$, with an initial parameter $\theta_0$. Then at every training episode $k$ we observe the state after applying the $0$-th step policy, $\pi_{\theta,0}$, Line \ref{line:observe_x1}. Then at every iteration $i$ of the episode $k$, the parameter $\theta_i$ of the policy is optimized.

Given a concrete-time reward (Def \ref{def:concrete-time-reward}), we introduce $J^{k}_{\phi}(\cdot,\cdot,\cdot,\cdot)$, a variant of the clipped objective (\ref{eq:L_k-Clipped_Objective}) defined based on $A_{\phi,t}^{k}$ and is given as $J^k_{\phi}(x,u,\theta,\theta_i) :=$ $\min \left[ \frac{\pi_{\theta}(u|x)}{\pi_{\theta_i}(u|x)} \hat{A}_{\phi,t}^{\pi_{\theta_i\cdot\rho},k},g(\epsilon, \hat{A}_{\phi,t}^{\pi_{\theta_i\cdot\rho},k})\right]$. Consequently, the parameterized part of the policy update is computed according to  $$\theta_{i+1}\gets\argmax\limits_{\theta\in\Theta}\; \mathop{\mathrm{E}}[J^k_{\phi}(x,u,\theta,\theta_i)]$$.
\begin{algorithm}[H]
        \textbf{Input} $\pi_{\rho}(\cdot)$ a fixed task predictor trained offline\\
        \textbf{Initialize}$\{\pi_{\theta,t}\}_{t=0}^{N-1}$ with $\theta = \theta_0$\label{line:init}\\
        \ForEach{episode $k\in\{1,\;\dots,\;K\}$}
            {Observe $x_1$\label{line:observe_x1}\\
                \ForEach{iteration $i\in\{1,\;\dots,\;N\}$}
                    {\label{line:begin_optimization}
                            $\theta_{i+1}^{k}\gets \argmax\limits_{\theta\in\Theta}\; \mathop{\mathrm{E}}[J^k_{\phi}(x,u,\theta,\theta_i)]$\\
                       Apply the control $u_i\sim\pi_{\theta_i}$\\
                       Compute $\hat{A}_{\phi,t}^{\pi_{\theta_i,t},k}$ according to (\ref{eq:A_hat})\label{line:end_optimization}
                    }
            }
   \caption{Accelerated PPO}\label{alg:APPO}
\end{algorithm}

\begin{figure}[t]
    \centering
    \begin{minipage}[t]{0.43\textwidth}
        \centering
        \includegraphics[width=\textwidth]{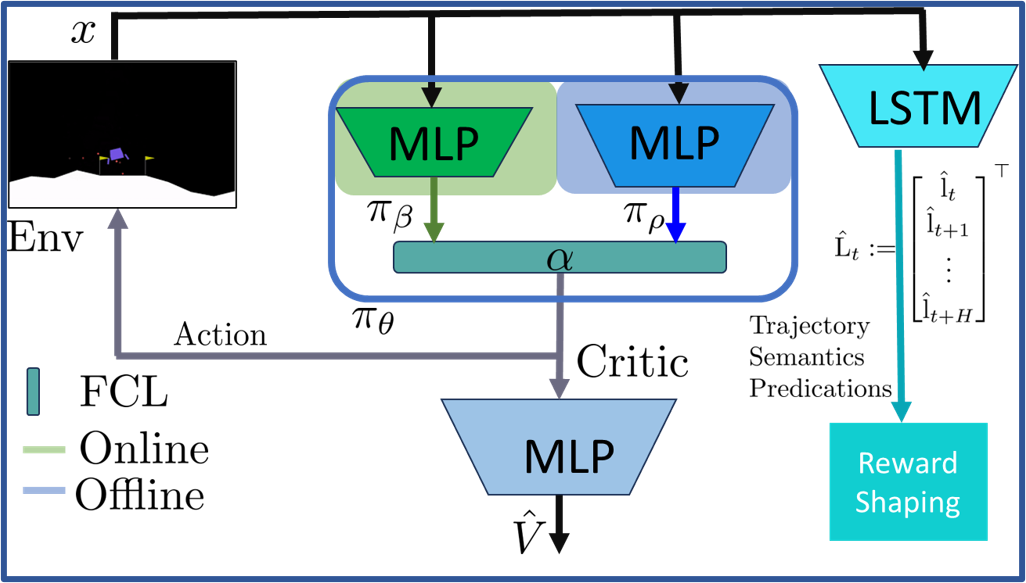}
        \caption{The Actor-Critic RL framework. The Actor's architecture, policy $\pi_\theta$, consists of an offline policy, $\pi_\rho$, and an adaptive policy, $\pi_\beta$, where the two policies are mixed using the parameters of the FCL, $\alpha$. The critic consists of an MLP that approximates the value function. The task predictor LSTM network and the reward shaping are depicted in cyan.}
        \label{fig:WholeSchemeArch}
    \end{minipage}%
    \hfill
    \begin{minipage}[t]{0.57\textwidth}
        \centering
        \includegraphics[width=\textwidth]{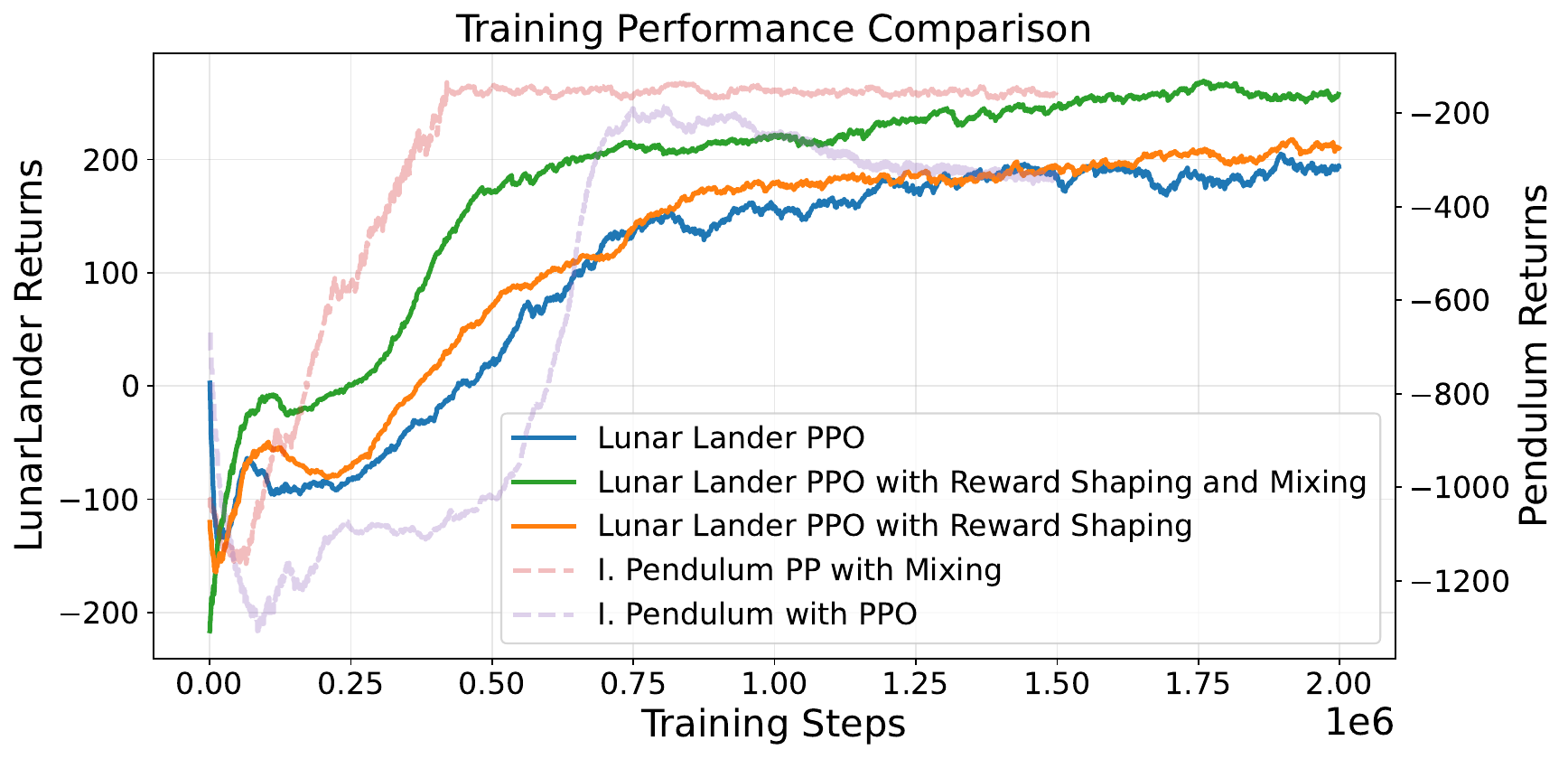}
        \caption{Training performance comparison across different variants of PPO in LunarLander-v2 and Pendulum environments. For LunarLander-v2 (left $y$-axis), we compare vanilla PPO against variants with reward shaping and mixing. The results show that combining reward shaping with policy mixing achieves faster learning and better performance compared to baseline PPO and reward shaping alone. For Pendulum (right $y$-axis, dashed lines), we include PPO with and without mixing for reference. Training steps are shown in millions on the $x$-axis.}
        \label{fig:resultsALL}
    \end{minipage}
\end{figure}

\section{Case Study}\label{sec:caseStudies}

As a continuation of Examples \ref{ex:lunarLanderTWTLLandingTask} and \ref{ex:LSTM_for_LunarLander}, we demonstrate the empirical performance of our approach in Figure \ref{fig:resultsALL}. For the Lunar Lander environment, where we previously defined the landing task TWTL formula $\phi_{\mathrm{landing}}$ and the LSTM-based task predictor, combining policy mixing with reward shaping (green) achieves faster learning and better asymptotic performance compared to both vanilla PPO (blue) and reward shaping alone (orange). The improvement is particularly noticeable in the early training phase (0-0.5M steps), where our method accelerates learning. 

To test the robustness of our mixing approach, we intentionally degraded pre-trained PPO policies by adding controlled noise to their parameters, using these as our offline policies $\pi_\rho$. Despite starting with these suboptimal policies, our method successfully leverages their partial knowledge while learning to improve upon them. For comparison, we also tested our approach on the Inverted Pendulum environment with similarly degraded offline policies (right $y$-axis, dashed lines), where policy mixing again demonstrates improved learning dynamics over the baseline PPO implementation.


\section{Conclusion and Future Work}\label{sec:conclusion}

We presented an approach to accelerate reinforcement learning in environments with delayed rewards through two key innovations: a mixed policy architecture that combines potentially suboptimal offline policies with online learning, and a TWTL-based reward shaping mechanism enabled by a task predictor. For the mixed policy architecture, we established theoretical guarantees showing consistent improvement over both the offline policy and previous iterations. Our empirical results on the Lunar Lander and Inverted Pendulum environments demonstrate that even with degraded offline policies, our method achieves faster learning and better asymptotic performance compared to vanilla PPO.

Several directions for future work emerge from this study. First, exploring more complex TWTL specifications that capture intricate temporal dependencies could extend our framework to more challenging tasks, such as soccer games. Second, investigating adaptive mixing strategies that automatically adjust based on the relative performance of offline and online policies could further improve learning efficiency. The integration of other temporal logic formalisms and their corresponding quantitative semantics could also provide interesting avenues for reward shaping in reinforcement learning.

\bibliography{IEEEabrv,l4dc2025-sample}

\begin{thebibliography}{31}
\providecommand{\natexlab}[1]{#1}
\providecommand{\url}[1]{\texttt{#1}}
\expandafter\ifx\csname urlstyle\endcsname\relax
  \providecommand{\doi}[1]{doi: #1}\else
  \providecommand{\doi}{doi: \begingroup \urlstyle{rm}\Url}\fi

\bibitem[Ahmad et~al.(2023)Ahmad, Vasile, Tron, and Belta]{ahmad2023TWTLrobustness}
Ahmad Ahmad, Cristian-Ioan Vasile, Roberto Tron, and Calin Belta.
\newblock Robustness measures and monitors for time window temporal logic.
\newblock In \emph{2023 62nd IEEE Conference on Decision and Control (CDC)}, pages 6841--6846. IEEE, 2023.

\bibitem[Aksaray et~al.(2016)Aksaray, Jones, Kong, Schwager, and Belta]{Belta16_QLwithSTL}
Derya Aksaray, Austin Jones, Zhaodan Kong, Mac Schwager, and Calin Belta.
\newblock Q-learning for robust satisfaction of signal temporal logic specifications.
\newblock In \emph{2016 IEEE 55th Conference on Decision and Control (CDC)}, pages 6565--6570. IEEE, 2016.

\bibitem[Alshiekh et~al.(2018)Alshiekh, Bloem, Ehlers, K{\"o}nighofer, Niekum, and Topcu]{ufuq2018safeSheildingRL}
Mohammed Alshiekh, Roderick Bloem, R{\"u}diger Ehlers, Bettina K{\"o}nighofer, Scott Niekum, and Ufuk Topcu.
\newblock Safe reinforcement learning via shielding.
\newblock In \emph{Proceedings of the AAAI conference on artificial intelligence}, volume~32, 2018.

\bibitem[Asarkaya et~al.(2021)Asarkaya, Aksaray, and Yaz{\i}c{\i}o{\u{g}}lu]{asarkaya2021twtlRL}
Ahmet~Semi Asarkaya, Derya Aksaray, and Yasin Yaz{\i}c{\i}o{\u{g}}lu.
\newblock Temporal-logic-constrained hybrid reinforcement learning to perform optimal aerial monitoring with delivery drones.
\newblock In \emph{2021 International Conference on Unmanned Aircraft Systems (ICUAS)}, pages 285--294. IEEE, 2021.

\bibitem[Baier and Katoen(2008)]{baier2008principles}
Christel Baier and Joost-Pieter Katoen.
\newblock \emph{Principles of model checking}.
\newblock MIT press, 2008.

\bibitem[Balakrishnan and Deshmukh(2019)]{jyo2019STL_rewards}
Anand Balakrishnan and Jyotirmoy~V Deshmukh.
\newblock Structured reward shaping using signal temporal logic specifications.
\newblock In \emph{2019 IEEE/RSJ International Conference on Intelligent Robots and Systems (IROS)}, pages 3481--3486. IEEE, 2019.

\bibitem[Ball et~al.(2023)Ball, Smith, Kostrikov, and Levine]{Levine23USeOffineDataWithRL}
Philip~J Ball, Laura Smith, Ilya Kostrikov, and Sergey Levine.
\newblock Efficient online reinforcement learning with offline data.
\newblock In \emph{International Conference on Machine Learning}, pages 1577--1594. PMLR, 2023.

\bibitem[Cai et~al.(2023)Cai, Aasi, Belta, and Vasile]{Cristi2022overcomingExpl}
Mingyu Cai, Erfan Aasi, Calin Belta, and Cristian-Ioan Vasile.
\newblock Overcoming exploration: Deep reinforcement learning for continuous control in cluttered environments from temporal logic specifications.
\newblock \emph{IEEE Robotics and Automation Letters}, 8\penalty0 (4):\penalty0 2158--2165, 2023.

\bibitem[Cai et~al.(2020)Cai, Yang, Jin, and Wang]{ICML20effExpl_inPO_OPPO}
Qi~Cai, Zhuoran Yang, Chi Jin, and Zhaoran Wang.
\newblock Provably efficient exploration in policy optimization.
\newblock In Hal~Daumé III and Aarti Singh, editors, \emph{Proceedings of the 37th International Conference on Machine Learning}, volume 119 of \emph{Proceedings of Machine Learning Research}, pages 1283--1294. PMLR, 13--18 Jul 2020.
\newblock URL \url{https://proceedings.mlr.press/v119/cai20d.html}.

\bibitem[Chen et~al.(2021)Chen, Lu, Rajeswaran, Lee, Grover, Laskin, Abbeel, Srinivas, and Mordatch]{chen2021decision}
Lili Chen, Kevin Lu, Aravind Rajeswaran, Kimin Lee, Aditya Grover, Misha Laskin, Pieter Abbeel, Aravind Srinivas, and Igor Mordatch.
\newblock Decision transformer: Reinforcement learning via sequence modeling.
\newblock \emph{Advances in neural information processing systems}, 34:\penalty0 15084--15097, 2021.

\bibitem[Grudzien et~al.(2022)Grudzien, De~Witt, and Foerster]{2022mirrorLearningPPO}
Jakub Grudzien, Christian A~Schroeder De~Witt, and Jakob Foerster.
\newblock Mirror learning: A unifying framework of policy optimisation.
\newblock In \emph{International Conference on Machine Learning}, pages 7825--7844. PMLR, 2022.

\bibitem[Hu et~al.(2023)Hu, Mirchandani, and Sadigh]{dorsa2023imitation}
Hengyuan Hu, Suvir Mirchandani, and Dorsa Sadigh.
\newblock Imitation bootstrapped reinforcement learning.
\newblock \emph{arXiv preprint arXiv:2311.02198}, 2023.

\bibitem[Hunter and Lange(2004)]{hunter2004tutorialMM}
David~R Hunter and Kenneth Lange.
\newblock A tutorial on mm algorithms.
\newblock \emph{The American Statistician}, 58\penalty0 (1):\penalty0 30--37, 2004.

\bibitem[Icarte et~al.(2022)Icarte, Klassen, Valenzano, and McIlraith]{rewardMachine2022reward}
Rodrigo~Toro Icarte, Toryn~Q Klassen, Richard Valenzano, and Sheila~A McIlraith.
\newblock Reward machines: Exploiting reward function structure in reinforcement learning.
\newblock \emph{Journal of Artificial Intelligence Research}, 73:\penalty0 173--208, 2022.

\bibitem[Kakade and Langford(2002)]{kakade2002approximately}
Sham Kakade and John Langford.
\newblock Approximately optimal approximate reinforcement learning.
\newblock In \emph{Proceedings of the Nineteenth International Conference on Machine Learning}, pages 267--274, 2002.

\bibitem[Kumar et~al.(2020)Kumar, Zhou, Tucker, and Levine]{kumar2020conservative}
Aviral Kumar, Aurick Zhou, George Tucker, and Sergey Levine.
\newblock Conservative q-learning for offline reinforcement learning.
\newblock \emph{Advances in Neural Information Processing Systems}, 33:\penalty0 1179--1191, 2020.

\bibitem[Li et~al.(2017)Li, Vasile, and Belta]{Belta2017TLTL}
Xiao Li, Cristian-Ioan Vasile, and Calin Belta.
\newblock Reinforcement learning with temporal logic rewards.
\newblock In \emph{2017 IEEE/RSJ International Conference on Intelligent Robots and Systems (IROS)}, pages 3834--3839. IEEE, 2017.

\bibitem[Nair et~al.(2020)Nair, Gupta, Dalal, and Levine]{levine2020awac}
Ashvin Nair, Abhishek Gupta, Murtaza Dalal, and Sergey Levine.
\newblock Awac: Accelerating online reinforcement learning with offline datasets.
\newblock \emph{arXiv preprint arXiv:2006.09359}, 2020.

\bibitem[Neider et~al.(2021)Neider, Gaglione, Gavran, Topcu, Wu, and Xu]{Ufuq2021advice_guided}
Daniel Neider, Jean-Raphael Gaglione, Ivan Gavran, Ufuk Topcu, Bo~Wu, and Zhe Xu.
\newblock Advice-guided reinforcement learning in a non-markovian environment.
\newblock In \emph{Proceedings of the AAAI Conference on Artificial Intelligence}, volume~35, pages 9073--9080, 2021.

\bibitem[Ng et~al.(1999)Ng, Harada, and Russell]{ng1999rewardShaping}
Andrew~Y Ng, Daishi Harada, and Stuart Russell.
\newblock Policy invariance under reward transformations: Theory and application to reward shaping.
\newblock In \emph{Icml}, volume~99, pages 278--287, 1999.

\bibitem[Ravari et~al.(2024)Ravari, Ghoreishi, and Imani]{ravari2024implicit}
Amirhossein Ravari, Seyede~Fatemeh Ghoreishi, and Mahdi Imani.
\newblock Implicit human perception learning in complex and unknown environments.
\newblock In \emph{American Control Conference (ACC)}, 2024.

\bibitem[Ross and Bagnell(2012)]{ross2012agnostic}
St{\'e}phane Ross and J~Andrew Bagnell.
\newblock Agnostic system identification for model-based reinforcement learning.
\newblock In \emph{Proceedings of the 29th International Coference on International Conference on Machine Learning}, pages 1905--1912, 2012.

\bibitem[Salzmann et~al.(2020)Salzmann, Ivanovic, Chakravarty, and Pavone]{pavone2020trajectron++}
Tim Salzmann, Boris Ivanovic, Punarjay Chakravarty, and Marco Pavone.
\newblock Trajectron++: Dynamically-feasible trajectory forecasting with heterogeneous data.
\newblock In \emph{Computer Vision--ECCV 2020: 16th European Conference, Glasgow, UK, August 23--28, 2020, Proceedings, Part XVIII 16}, pages 683--700. Springer, 2020.

\bibitem[Schmitt et~al.(2018)Schmitt, Hudson, Zidek, Osindero, Doersch, Czarnecki, Leibo, Kuttler, Zisserman, Simonyan, et~al.]{schmitt2018kickstarting}
Simon Schmitt, Jonathan~J Hudson, Augustin Zidek, Simon Osindero, Carl Doersch, Wojciech~M Czarnecki, Joel~Z Leibo, Heinrich Kuttler, Andrew Zisserman, Karen Simonyan, et~al.
\newblock Kickstarting deep reinforcement learning.
\newblock \emph{arXiv preprint arXiv:1803.03835}, 2018.

\bibitem[Schulman et~al.(2015{\natexlab{a}})Schulman, Levine, Abbeel, Jordan, and Moritz]{abbele2015TRPO}
John Schulman, Sergey Levine, Pieter Abbeel, Michael Jordan, and Philipp Moritz.
\newblock Trust region policy optimization.
\newblock In \emph{International conference on machine learning}, pages 1889--1897. PMLR, 2015{\natexlab{a}}.

\bibitem[Schulman et~al.(2015{\natexlab{b}})Schulman, Moritz, Levine, Jordan, and Abbeel]{abbeel16Adv_estimate}
John Schulman, Philipp Moritz, Sergey Levine, Michael Jordan, and Pieter Abbeel.
\newblock High-dimensional continuous control using generalized advantage estimation.
\newblock \emph{arXiv preprint arXiv:1506.02438}, 2015{\natexlab{b}}.

\bibitem[Schulman et~al.(2017)Schulman, Wolski, Dhariwal, Radford, and Klimov]{ppo}
John Schulman, Filip Wolski, Prafulla Dhariwal, Alec Radford, and Oleg Klimov.
\newblock Proximal policy optimization algorithms.
\newblock \emph{arXiv preprint arXiv:1707.06347}, 2017.

\bibitem[Sutton and Barto(2018)]{sutton2018reinforcement}
Richard~S Sutton and Andrew~G Barto.
\newblock \emph{Reinforcement learning: An introduction}.
\newblock MIT press, 2018.

\bibitem[Towers et~al.(2023)Towers, Terry, Kwiatkowski, Copplestone-Bruce, Lutati, Dossa, Famularo, Pytlak, Ballester, Guo, Hu, and Awad]{gymnasium2023}
Mark Towers, Jordan~K. Terry, Ariel Kwiatkowski, Andrew Copplestone-Bruce, Gianluca Lutati, Rousslan Fernand~Julien Dossa, Antonin Famularo, Filip Pytlak, Marc Ballester, Hao Guo, Jiayi Hu, and Mikayel Awad.
\newblock Gymnasium.
\newblock \url{https://github.com/Farama-Foundation/Gymnasium}, 2023.
\newblock URL \url{https://gymnasium.farama.org/}.

\bibitem[Vasile et~al.(2017)Vasile, Aksaray, and Belta]{Cristi2017TWTL}
Cristian-Ioan Vasile, Derya Aksaray, and Calin Belta.
\newblock Time window temporal logic.
\newblock \emph{Theoretical Computer Science}, 691:\penalty0 27--54, 2017.

\bibitem[Xu et~al.(2020)Xu, Gavran, Ahmad, Majumdar, Neider, Topcu, and Wu]{ufuq2020rewardMachine}
Zhe Xu, Ivan Gavran, Yousef Ahmad, Rupak Majumdar, Daniel Neider, Ufuk Topcu, and Bo~Wu.
\newblock Joint inference of reward machines and policies for reinforcement learning.
\newblock In \emph{Proceedings of the International Conference on Automated Planning and Scheduling}, volume~30, pages 590--598, 2020.

\end{thebibliography}

\end{document}